\documentclass[11pt,reqno]{amsart}

\usepackage[utf8]{inputenc}
\usepackage[english]{babel}
\usepackage{latexsym}
\usepackage{amssymb,amsfonts}
\usepackage{mathtools}
\usepackage{amsthm}
\usepackage{amsmath}
\usepackage{mathrsfs}
\usepackage{amstext}
\usepackage{graphicx}  
\usepackage{amsfonts}
\usepackage{pdfpages}
\usepackage{epic}
\usepackage{fullpage}
\usepackage{color}
\usepackage{curves}
\usepackage{subfigure}
\usepackage{setspace}
\usepackage{anysize}
\usepackage{rotating}
\usepackage{enumerate} 
\usepackage{hyperref}
\usepackage{bbm}
\usepackage{bm}
\usepackage{float}
\usepackage{microtype}
\usepackage{multirow}
\usepackage{makecell}
\usepackage{wrapfig}

\DisableLigatures{encoding = *, family = * }

\numberwithin{equation}{section}

\newtheorem{theorem}{Theorem}[section] 
\newtheorem{definition}[theorem]{Definition}
\newtheorem{corollary}[theorem]{Corollary}
\newtheorem{proposition}[theorem]{Proposition}
\newtheorem{lemma}[theorem]{Lemma}

\def\R{{\mathbb R}}

\newcommand{\MC}[1]{\mathcal{#1}}

\title{Fair feature attribution for multi-output prediction: a Shapley-based perspective}

\author{Umberto Biccari\textsuperscript{\,$\ast$}} 
\address{\textsuperscript{$\ast$}\, Chair of Computational Mathematics, DeustoTech, University of Deusto, Avenida de las Universidades 24, 48007 Bilbao, Basque Country, Spain
         \newline \indent \textsuperscript{$\dagger$}\, Chair for dynamics, control, machine learning, and numerics (Alexander Von Humboldt-Professorship), Department of Mathematics, Friedrich-Alexander-Universit\"at Erlangen-N\"urnberg, 91058, Erlangen, Germany.
         \newline \indent \textsuperscript{$\S$}\, Universidad Aut\'onoma de Madrid, Departamento de Matem\'aticas, Ciudad Universitaria de Cantoblanco, 28049 Madrid, Spain.
         \newline \indent \textsuperscript{$\ddagger$}\, ATLAS Molecular Pharma, Parque Tecnol\'ogico de Bizkaia, Ed. 800, 48160 Derio, Spain.
         \newline \indent \textsuperscript{$\P$}\, Precision Medicine and Metabolism Laboratory, CIC bioGUNE, Basque Research and Technology Alliance, Parque Tecnol\'ogico de Bizkaia, Ed. 800, 48160 Derio, Spain.
         \newline \indent \textsuperscript{$\|$}\, CIBERehd, Madrid, 28025, Spain.
} 

\email{umberto.biccari@deusto.es}
\thanks{This project has received funding from the European Research Council (ERC) under the European Union's Horizon 2030 research and innovation programme (grant agreement NO: 101096251-CoDeFeL). EZ was partially supported by the Alexander von Humboldt Professorship program; the European Union’s Horizon Europe MSCA project ModConFlex (HORIZON-MSCA-2021-DN-01(project 101073558); the Transregio 154 Project “Mathematical Modelling, Simulation and Optimization Using the Example of Gas Networks” of the DFG; the AFOSR 24IOE027 project; SURE-AI: The Norwegian Centre for Sustainable, Risk-Averse, and Ethical AI grant 357482, Research Council of Norway; and the Madrid Government–UAM Agreement for the Excellence of the University Research Staff in the context of the V PRICIT (Regional Programme of Research and Technological Innovation). UB and EZ were partially supported by the Grant PID2023-146872OB-I00-DyCMaMod of MICIU (Spain) and by the COST Action CA24122 - multiscale Stochastics, Patterns, and Analysis of Combinatorial Environments. UB and RM were partially supported by the COST Action CA24136 - Interactions between Control Theory and Machine Learning. AI, JMM and \'OM were supported by the Elkartek grants bg2023 and bg2025 of the Basque Government and by the grants PID2024-160696OB-I00, CPP2023-01047 and CPP2022-009738 of MICIU (Spain)}

\author{Alain Ib\'a\~{n}ez de Opakua\textsuperscript{\,$\ddagger$}}
\email{aibanez@cicbiogune.es}

\author{Jos\'e Mar\'ia Mato\textsuperscript{\,$\P$\,$\|$}}
\email{director@cicbiogune.es}

\author{\'Oscar Millet\textsuperscript{\,$\ddagger$\,$\P$\,$\|$}}
\email{omillet@cicbiogune.es}

\author{Roberto Morales\textsuperscript{\,$\ast$}}  
\email{roberto.morales@deusto.es}

\author{Enrique Zuazua\textsuperscript{\,$\ast$\,$\dagger$\,$\S$}}  
\email{enrique.zuazua@fau.de, enrique.zuazua@deusto.es, enrique.zuazua@uam.es}

\begin{document}

\begin{abstract}
In this article, we provide an axiomatic characterization of feature attribution for multi-output predictors within the Shapley framework. While SHAP explanations are routinely computed independently for each output coordinate, the theoretical necessity of this practice has remained unclear. By extending the classical Shapley axioms to vector-valued cooperative games, we establish a rigidity theorem showing that any attribution rule satisfying efficiency, symmetry, dummy player, and additivity must necessarily decompose component-wise across outputs. Consequently, any joint-output attribution rule must relax at least one of the classical Shapley axioms. This result identifies a previously unformalized structural constraint in Shapley-based interpretability, clarifying the precise scope of fairness-consistent explanations in multi-output learning. 
Numerical experiments on a biomedical benchmark illustrate that multi-output models can yield computational savings in training and deployment, while producing SHAP explanations that remain fully consistent with the component-wise structure imposed by the Shapley axioms.
\end{abstract}

\maketitle

\section{Introduction and motivations}

Feature attribution is a central problem in the interpretability of Machine Learning (ML) models: given a prediction, how should the contribution of each input feature be fairly and consistently quantified? This question is especially relevant in high-stakes applications such as healthcare, finance, and public decision-making \cite{aldughayfiq2023explainable,arrieta2020explainable,bernard2023explainable,samek2017explainable}, and is explicitly reflected in regulatory frameworks like the European Union's GDPR, which requires transparency and ``meaningful information about the logic involved'' in algorithmic decisions \cite[Articles 13-15]{GDPR}.

Among existing approaches, SHAP (SHapley Additive exPlanations, \cite{arunika2024survey,lundberg2017unified,molnar2020interpretable,mosca2022shap,zhang2021survey}) has emerged as a reference framework due to its axiomatic foundation in the Shapley value from cooperative game theory \cite{shapley1953value}. This axiomatic foundation provides a principled notion of fairness and consistency for feature attribution and distinguishes SHAP from heuristic or model-dependent interpretability techniques.

While the theory of SHAP is fully understood for scalar-valued predictors, modern ML applications increasingly rely on multi-output models $f:\mathbb{R}^n \to \mathbb{R}^m$, which jointly predict several correlated targets \cite{harutyunyan2019multitask,shao2020multi,zhang2018deep}. In such settings, SHAP is routinely applied by computing explanations independently for each output coordinate. This practice is widespread and computationally convenient; however, it lacks a rigorous axiomatic justification: it remains unclear whether component-wise attribution is a modeling convention or a structural consequence of the Shapley principles.

Given this premise, the central question addressed in this paper is the following: do the classical Shapley axioms admit feature attribution rules for vector-valued predictors that couple different output coordinates, or do they instead force output-wise separation?

Our main theoretical contribution answers this question. In fact, Theorem \ref{thm:coordinatewise_decomposition} shows that any Shapley-consistent attribution rule is uniquely determined and must decompose component-wise across outputs. In particular, no joint-output Shapley explanation can exist without abandoning at least one of the classical axioms. For clarity, we summarize this contribution in the following informal statement.

\medskip
\noindent\textbf{Main result (informal).}
\emph{
Let $f:\mathbb{R}^n \to \mathbb{R}^m$ be a multi-output predictor and model its feature attribution problem as a vector-valued cooperative game. If an attribution rule satisfies the classical Shapley axioms extended to vector-valued payoffs, then it is uniquely determined and necessarily decomposes component-wise across outputs. In particular, there exists no joint-output Shapley attribution rule: any method that couples different outputs must necessarily violate at least one of the classical Shapley axioms.
}
\medskip

At first sight, one might conjecture that extending Shapley values to vector-valued predictors simply amounts to applying the scalar theory coordinate-wise. However, this conclusion is not a formal consequence of the classical axioms. In the vector-valued setting \cite{Allal2025,mosca2022shap,zhang2021survey}, the payoff space admits linear couplings across coordinates, and therefore, a priori, attribution rules could redistribute contributions between outputs while still satisfying symmetry or dummy conditions.

The main novelty of our result is to show that this apparent flexibility is incompatible with the classical Shapley axioms extended to $\mathbb{R}^m$: once efficiency and additivity are imposed directly in the vector space $\mathbb{R}^m$, any cross-coordinate coupling becomes impossible. Thus, component-wise SHAP explanations are not merely a convenient convention, but the unique consequence of extending the classical Shapley axioms verbatim to multi-output prediction.

To our knowledge, the rigidity phenomenon has not been made explicit in the multi-output interpretability literature unless the direct extension of the classical Shapley axioms to vector-valued pay-offs.


\subsection{Mathematical origins of SHAP analysis and the axiomatic approach to interpretability}

SHAP is grounded in cooperative game theory, where the problem of fairly allocating a collective payoff among individual contributors was formalized by Shapley in his seminal work \cite{shapley1953value}. In this framework, a finite set of players forms coalitions, each coalition being assigned a numerical value representing the payoff it can achieve through cooperation. The central question is how to distribute the value of the grand coalition among players in a manner that is fair, consistent, and independent of arbitrary conventions.

Shapley addressed this question through an axiomatic characterization of allocation rules. The resulting Shapley value is uniquely determined by four principles: \textit{efficiency}, requiring that the total payoff be exactly distributed among players; \textit{symmetry}, requiring identical treatment of players with identical contributions; the \textit{dummy player} axiom, assigning zero value to players that contribute nothing; and \textit{additivity}, enforcing linear consistency with respect to the superposition of independent games. These axioms uniquely determine a single allocation rule, which admits explicit representations in terms of average marginal contributions over all coalitions.

This axiomatic construction was later adapted to model interpretability in ML, giving rise to SHAP values \cite{lundberg2017unified}. In this setting, input features play the role of players, and the characteristic function of the game is defined via conditional expectations of the model output. Feature attributions are then computed as Shapley values of the induced game, ensuring that explanations satisfy the same fairness and consistency principles as in the original cooperative framework.

The appeal of SHAP lies precisely in this axiomatic grounding. Interpretability in ML admits many possible operational definitions, often leading to method-dependent or model-dependent explanations \cite{arrieta2020explainable,molnar2020interpretable,samek2017explainable}. An axiomatic approach shifts the focus from ``how explanations are computed'' to ``which properties they are required to satisfy'', providing a principled and model-agnostic foundation for feature attribution. This perspective underlies the analysis developed in the present work and motivates the extension of Shapley theory to the multi-output setting considered in the following sections.

\subsection{Biomedical motivation for multi-output modeling and interpretability} 

The present work is motivated by challenges arising in biomedical ML, where interpretability is often essential for scientific plausibility, clinical relevance, and downstream decision-making. In this domain, predictive models are routinely used to support biological discovery and clinical research, and feature attributions are expected to be interpretable in light of existing physiological or biochemical knowledge \cite{aldughayfiq2023explainable,bernard2023explainable,vellido2020importance}.

Biomedical prediction problems naturally give rise to multi-output models, as biological systems are characterized by collections of correlated and interdependent quantities \cite{moreira2019comprehensive}. A representative example is provided by metabolomic and biological aging clocks \cite{ibanez2025metabolomic,wen2025refining}. While foundational models in this domain typically operate as single-output predictors estimating a global biological age, recent advances have shifted towards multi-output architectures to better capture physiological complexity. For instance, modern approaches deploy multi-task learning to predict aging across multiple organ systems simultaneously (e.g. brain, heart, and liver) rather than a single systemic age \cite{Tian2023HeterogeneousAA}. Similarly, joint prediction models incorporate auxiliary tasks to improve latent representation robustness \cite{Eraslan2019DeepLN,Zhou2013ModelingDP}. Furthermore, deep learning frameworks in multi-omics are increasingly used to forecast entire trajectories of biomarker panels jointly \cite{Song2018DeepR}. 


In the context of metabolomic and aging-related prediction tasks, several composite indices have been proposed to summarize systemic metabolic dysregulation. Among them, the MetSCORE index \cite{gil2024metscore} has been recently introduced as a continuous, metabolomics-based score designed to quantify metabolic risk by integrating information from lipid, amino acid, and energy metabolism pathways. Unlike classical binary or threshold-based metabolic syndrome definitions, MetSCORE provides a smooth and data-driven risk measure that is well suited for regression-based learning and interpretability analyses. This makes it particularly attractive as an output variable in multi-output predictive models, where heterogeneous yet biologically related targets are jointly learned.

In these settings, feature attribution plays a central role. Identifying which variables contribute to changes in specific biomarkers or physiological indicators is essential for biological interpretation, hypothesis generation, and validation against prior knowledge. At the same time, the presence of multiple correlated outputs raises nontrivial interpretability questions, particularly regarding whether and how explanations should reflect inter-output dependencies.

Although the theoretical results developed in this paper are fully general, biomedical applications provide a concrete and practically relevant setting in which multi-output prediction and interpretability are both indispensable. For this reason, the numerical experiments in Section \ref{section:Numerical:experiments} focus on a representative biomedical multi-output regression problem, which is used to illustrate the computational and interpretative implications of Shapley-based feature attribution in the multi-output setting.


\subsection{SHAP in context: comparison with related interpretability methods}

A broad range of interpretability methods has been proposed in the ML literature, reflecting the diversity of notions of explanation and attribution. These approaches differ not only in how feature relevance is computed, but also in the implicit structural or axiomatic assumptions they make when dealing with multiple outputs or objectives.

A first class of methods consists of \textbf{scalarization}- or \textbf{preference}-\textbf{based approaches}, in which feature relevance is derived from local sensitivities or from aggregating effects along prescribed paths or baselines. Local sensitivity methods \cite{simonyan2013deep,selvaraju2017grad} attribute importance to input features using derivatives of the model output with respect to the input, thereby privileging infinitesimal variations around a given point. Closely related path-integrated and reference-based methods \cite{sundararajan2017axiomatic,shrikumar2017learning} accumulate such sensitivities along a chosen path from a baseline input. In both cases, vector-valued effects are implicitly reduced to scalar quantities through choices of reference points, paths, or scaling conventions, which act as preference structures.

A second family includes \textbf{set}-\textbf{valued} or \textbf{Pareto}-\textbf{type approaches}, which assess feature relevance through systematic perturbations or local approximations rather than enforcing a unique additive allocation. In this context, perturbation methods \cite{zeiler2014visualizing,ribeiro2016should} evaluate importance by measuring changes in model outputs under feature removal, masking, or by fitting local surrogate models. These approaches are typically model-agnostic but depend on sampling strategies and locality assumptions, and therefore do not yield a unique, globally additive attribution across features.

Finally, a third class is formed by \textbf{redistribution or cost}-\textbf{sharing variants}, which explain predictions by propagating or redistributing relevance through the internal structure of a model. Propagation and decomposition methods \cite{bach2015pixel,montavon2017explaining,zhang2018top} allocate relevance across layers according to predefined conservation rules, often tailored to specific architectures. Within this family, attention-based methods \cite{abnar2020quantifying,clark2019does} interpret learned attention weights as indicators of feature importance, capturing patterns of information routing rather than enforcing full vector-valued efficiency or additivity.

A particularly prominent instance of this approach arises in Transformer architectures, where token importance is frequently inferred directly from attention weights. In such models, attention mechanisms define how information is mixed and propagated across tokens and layers through linear aggregation operators followed by nonlinear transformations and residual connections. However, several studies have demonstrated that attention scores do not necessarily reflect causal or marginal contribution to the final prediction. In particular, \cite{jain2019attention} shows that attention weights can be decorrelated from gradient-based importance measures without affecting model outputs, while \cite{serrano2019attention} observes that removing high-attention tokens does not always produce significant output changes. A more nuanced discussion is provided in \cite{wiegreffe2019attention}, where attention is argued to serve as an explanation only under additional structural assumptions. Quantitative analyses of attention flow across layers \cite{abnar2020quantifying} and circuit-based interpretations of Transformer architectures \cite{chefer2021transformer} further emphasize that attention primarily captures internal representation dynamics rather than global output contribution.

This distinction highlights a fundamental difference between information routing and output attribution. Attention mechanisms describe how representations are constructed internally, whereas Shapley-based methods quantify how input variables contribute to the realized prediction via conditional expectations of the predictor itself. In particular, SHAP values are uniquely characterized by axioms of efficiency, symmetry, dummy player, and additivity \cite{lundberg2017unified}, which impose global consistency constraints at the output level. This difference is especially relevant in multi-output settings. Attention mechanisms may implicitly couple outputs through shared latent representations, while our main result (see Theorem \ref{thm:coordinatewise_decomposition}) establishes that any attribution rule satisfying the classical Shapley axioms must remain output-separable. Hence, attention-based token importance and Shapley-based attribution address fundamentally different interpretability questions: the former concerns internal computational structure, whereas the latter concerns fair allocation of predictive contribution in the sense of cooperative game theory.

\subsection{Paper's contributions}
While vector-valued cooperative games and multicriteria allocation problems have been studied extensively in game theory, existing solution concepts typically rely on preference relations, scalarization procedures, or ordering cones in order to compare vector payoffs. These frameworks do not preserve the classical Shapley axioms verbatim in the ambient vector space.


The closest related works in the interpretability literature either apply SHAP independently to each output coordinate without axiomatic justification, or introduce scalarization or preference-based modifications that depart from the classical Shapley framework. None of these approaches establishes whether component-wise decomposition is a structural necessity or merely a modeling choice. The present paper resolves this question by proving that, under the classical Shapley axioms extended to vector-valued payoffs, output-wise separation is mathematically unavoidable.

\medskip 
\noindent In this framework, our main contributions are as follows this work are:
\begin{itemize}
    \item[(C1)] We formulate feature attribution for multi-output predictors as a vector-valued cooperative game and prove a rigidity theorem (Theorem \ref{thm:coordinatewise_decomposition}) showing that the classical Shapley axioms enforce output-wise separation.
    \item[(C2)] We derive a uniqueness characterization (Theorem \ref{thm:existence:uniqueness:Shapley:formula}) establishing that the only admissible vector-valued Shapley operator is the coordinate-wise application of the scalar Shapley value.
    \item[(C3)] We provide quantitative stability estimates for the vector-valued Shapley operator under perturbations of the underlying predictor (Proposition \ref{prop:estimates:Phi}).
    \item[(C4)] We illustrate the theoretical findings through a biomedical multi-output learning example, showing that multi-output modeling can reduce computational cost while preserving per-output SHAP explanations.
\end{itemize}

\subsection{Paper's organization}

The remainder of the paper is organized as follows. Section \ref{sec:SHAP_math} introduces the mathematical framework underlying Shapley values, recalling the classical scalar formulation and developing its extension to vector-valued settings. In particular, this section contains our main contributions that formalize the axiomatic structure of vector-valued Shapley operators and establish their fundamental properties. Section \ref{section:Numerical:experiments} presents numerical experiments illustrating our theoretical results. These experiments are used to examine the computational and interpretative behavior of SHAP explanations for multi-output models. Section \ref{sec:conclusion} concludes the paper with a summary of the main findings and a discussion of several open problems and directions for future research. Finally, Appendices \ref{appendix:proofs:main:results} and \ref{app:B} contain the technical proofs and additional mathematical details supporting the results presented in the main text.

\section{SHAP values for model interpretability}\label{sec:SHAP_math}

Given $m,n\in\mathbb N$, a predictor $f:\mathbb{R}^n \to \mathbb{R}^m$, and an input instance $x=(x_1,\ldots,x_n)\in\mathbb{R}^n$, feature attribution aims to quantify the contribution of each input variable $x_i$ to the prediction $f(x)$. In SHAP-based explanations, attributions are defined relative to a reference output, typically the expected model prediction under a background distribution. The objective is to decompose the deviation of $f(x)$ from this baseline into feature-wise contributions, in a manner that is independent of the specific model architecture and depends only on conditional expectations of the predictor.

A central aspect of SHAP is that feature contributions are evaluated relative to subsets of other features: the relevance of a variable depends on which additional variables are already known. This naturally leads to representing available information by subsets of the feature index set $[n]=\{1,\dots,n\}$, and associating to each subset a quantity capturing its predictive content.

This viewpoint admits a natural formulation in the language of cooperative game theory. Input features are interpreted as players, subsets of features as coalitions, and the predictive content of each coalition as a payoff. Feature attribution then corresponds to distributing the total predictive value among the players according to principled allocation rules.

The following subsections formalize this correspondence. Section \ref{subsec:SHAP_scalar} recalls the classical scalar Shapley framework, Section \ref{subsec:SHAP_vector} develops its extension to vector-valued payoffs, Section \ref{subsec:SHAP_discussion} contains an abridged discussion of the axiomatic structure of Shapley theory, and Section \ref{subsec:SHAP_ML} connects that theory to SHAP explanations for ML models.

\subsection{Scalar cooperative games and the Shapley value}\label{subsec:SHAP_scalar} 

We recall the classical cooperative game formulation underlying Shapley values in the scalar-output case \cite{shapley1953value}, which will serve as a reference framework for the vector-valued extension developed later.

Let $\MC P([n])$ denote the power set of $[n] = \{1,\dots,n\}$. In the context of cooperative game theory, each subset $S\in \MC P([n])$ represents a coalition of players joining forces to generate a collective outcome. To describe this process mathematically, we must specify, for every possible coalition, the total benefit it can obtain through cooperation. 

\begin{definition}\label{def:characteristic}
A \textbf{characteristic function} (or \textbf{score function}) is a map $v: \MC P([n]) \to \R$ which assigns to each coalition $S \in \MC P([n])$ a real number $v(S)$, representing the total payoff that the members of $S$ can jointly achieve by cooperating among themselves. A \textbf{cooperative game} is a pair $([n],v)$, where $v$ is a characteristic function over $\MC P([n])$.
\end{definition}

\noindent We denote by 
\begin{align}\label{def:mathcal:G:n}
    \MC G_{[n]}\coloneqq \Big\{v:\MC P([n])\to \R\,:\, v(\varnothing)=0\Big\},
\end{align}
the vector space of all characteristic functions satisfying the normalization condition $v(\varnothing) = 0$. This condition reflects the convention that no value is produced in the absence of players.

The central question in cooperative game theory is how to allocate the total payoff $v([n])$ among individual players in a fair and consistent manner. This motivates introducing the following notion of value operator.

\begin{definition}\label{def:value_operator}
A \textbf{value operator} on the players set $[n]$ is a mapping 
\begin{align*}
    \Phi:\MC G_{[n]}\to \R^n,\quad v\mapsto \big(\phi_1(v),\ldots, \phi_n(v)\big).
\end{align*}
that assigns to each game $v\in\MC G_{[n]}$ a vector of individual payoffs $(\phi_1(v), \ldots, \phi_n(v))$, where $\phi_i(v)\in\R$ represents the contribution attributed to player $i$.
\end{definition}

In broad terms, a value operator translates the abstract information contained in the characteristic function $v$ into a concrete allocation of gains among the players. To determine which allocation is most reasonable, Shapley proposed in \cite{shapley1953value} an axiomatic characterization of such value operators, formalizing minimal requirements of fairness and consistency.

\begin{definition}[Shapley operator]\label{def:SHAP_operator}
A value operator $\Phi: \MC G_{[n]} \to \mathbb{R}^n$ is called a \textbf{Shapley operator} if it satisfies the following axioms:
\begin{itemize}
    \item[(i)] \textbf{Efficiency.} The total payoff is fully distributed among the players: for all $v\in\MC G_{[n]}$,
    \begin{align*}
        \sum_{i\in [n]} \phi_i(v)=v([n]).
    \end{align*}
    \item[(ii)] \textbf{Symmetry.} If two players $i,j\in [n]$ contribute identically to every coalition, then they receive the same payoff: for all $v\in\MC G_{[n]}$,
    \begin{align*}
        v(S\cup \{i\})=v(S\cup \{j\})\quad\text{for all } S\subseteq \big([n]\setminus \{i,j\}\big) \Rightarrow \phi_i(v)=\phi_j(v).
    \end{align*}
    \item[(iii)] \textbf{Dummy (null player).} If a player $i\in [n]$ does not affect the value of any coalition, then its payoff is zero: for all $v\in\MC G_{[n]}$,
    \begin{align*}
        v(S\cup \{i\})=v(S)\quad \text{for all } S\subseteq \big([n]\setminus \{i\}\big) \Rightarrow \phi_i(v)=0.
    \end{align*}
    \item[(iv)] \textbf{Additivity.} The value operator is linear with respect to the game: for all $v,w\in \MC G_{[n]}$ and $\alpha,\beta\in \R$, $\Phi(\alpha v+\beta w)=\alpha \Phi(v) + \beta \Phi(w)$.
\end{itemize}
\end{definition}

The four Shapley axioms of Definition \ref{def:SHAP_operator} together ensure that the allocation rule behaves sensibly: the total reward is fully distributed (efficiency), identical players are treated identically (symmetry), irrelevant players receive nothing (dummy), and the rule is compatible with combining independent games (additivity). Remarkably, these simple principles are strong enough to determine a unique allocation rule, as established by the classical Shapley theorem.

\begin{theorem}[{\cite[Theorem 1]{shapley1953value}}]\label{thm:SHAP_scalar}
There exists a unique value operator $\Phi:\MC G_{[n]} \to\R^n$ satisfying the axioms of efficiency, symmetry, dummy, and additivity. For each game $v\in\MC G_{[n]}$ and player $i\in[n]$, this operator is given by 
\begin{align}\label{formula:Shapley:permu}
    \phi_i(v)=\dfrac{1}{n!} \sum_{\pi \in \mathfrak{S}_{[n]}} \Big[ v\big(P_i(\pi)\cup \{i\}\big) - v\big(P_i(\pi)\big) \Big],
\end{align}
where $\mathfrak{S}_{[n]}$ is the set of all possible permutations of the players $[n]$ and $P_i(\pi)$ denotes the set of players preceding $i$ in permutation $\pi\in\mathfrak{S}_{[n]}$.
\end{theorem}

Equivalently, the Shapley value admits a combinatorial representation as a weighted sum of marginal contributions over all coalitions not containing player $i$, where the weights depend only on the coalition size.

\begin{theorem}[{\cite[Section 4.2, page 59]{roth1988expected}}]\label{thm:1dim:Shapley:values}
For every game $v\in\MC G_{[n]}$ and player $i\in [n]$, the Shapley value $\phi_i(v)$ given by \ref{formula:Shapley:permu} admits the equivalent combinatorial representation 
\begin{align}\label{formula:permu:SHAP:1d}
    \phi_i(v)\coloneqq\sum_{S\subseteq ([n]\setminus \{i\})} \dfrac{|S|! (n-|S|-1)!}{n!} \Big[v(S\cup \{i\}) - v(S) \Big].
\end{align}
\end{theorem}

This formulation is particularly convenient for both theoretical analysis and algorithmic implementation, and it forms the basis of SHAP values in ML \cite{lundberg2017unified}.

\subsection{Shapley operators for vector-valued games}\label{subsec:SHAP_vector}

The scalar formulation reviewed in Section \ref{subsec:SHAP_scalar} applies to cooperative games in which each coalition receives a single real-valued payoff. In many modern applications, however, cooperation produces several outcomes simultaneously. This motivates the study of vector-valued cooperative games, in which each coalition is assigned a payoff in a finite-dimensional vector space. Although multi-output learning and vector-valued prediction have been studied \cite{Allal2025,Xu2020}, these works do not provide an axiomatic characterization of feature attribution rules preserving the classical Shapley principles in the full vector space. In particular, they do not address whether cross-output attribution rules are compatible with efficiency and additivity imposed directly in $\mathbb{R}^m$. The goal of this section is precisely to answer this structural question.

At first sight, extending Shapley values from scalar to vector-valued games may appear straightforward, by simply applying the scalar theory component-wise. However, this construction is not a priori justified by the axioms themselves: vector-valued payoffs allow, in principle, for attribution rules that mix or couple different output coordinates. The purpose of this section is to show that such couplings are in fact impossible under the Shapley axioms. To formalize this, for every $m \in \mathbb{N}$ let us denote by
\begin{align*}
    \MC{G}_{[n]}^m \coloneqq\Big\{v : \MC P([n])\to\R^m\,:\, v(\varnothing)=0\Big\}
\end{align*}
the space of vector-valued characteristic functions satisfying the normalization condition $v(\varnothing) = 0$. Each element of $\MC G^m_{[n]}$ represents a cooperative system producing $m$ different payoffs at once. 

We begin by defining the class of operators that assign vector-valued payoffs to individual players.

\begin{definition}[Vector-valued value operator]\label{def:operator_vector}
A \textbf{vector}-\textbf{valued value operator} on the player set $[n]$ is a mapping $\Phi : \MC G^m_{[n]}\to (\mathbb{R}^{m})^n$ that assigns to each game $v \in \MC G^m_{[n]}$ a collection of payoff vectors $(\phi_1(v), \ldots, \phi_n(v))$, where $\phi_i(v) \in \mathbb{R}^m$ represents the contribution attributed to player $i$.
\end{definition}

This Definition \ref{def:operator_vector} specifies the class of admissible allocation rules but does not yet impose any fairness or consistency requirements. To single out a meaningful notion of attribution, we now extend the classical Shapley axioms to the vector-valued setting.

\subsection{Shapley axioms in the vector-valued setting}
We now formalize the axiomatic principles that govern admissible feature attribution rules in the vector-valued setting.
\begin{definition}[Vector-valued Shapley operator]\label{def:SHAP_operator_vector}
A vector-valued value operator $\Phi : \MC G^m_{[n]}\to \mathbb{R}^{mn}$ is called a \textbf{vector}-\textbf{valued Shapley operator} if it satisfies the axioms in Definition \ref{def:SHAP_operator}, with all equalities interpreted componentwise in $\mathbb R^m$.
\end{definition}

This definition extends the classical Shapley axioms to vector-valued payoffs, now applied component-wise in the output space $\mathbb{R}^m$. In particular, efficiency and additivity impose strong structural constraints on admissible operators. 

At this point, it may be tempting to assume that extending Shapley values to vector-valued games simply amounts to applying the scalar theory component-wise. However, this conclusion is not a priori implied by the axioms. Vector-valued payoffs allow, in principle, attribution rules that redistribute payoff mass across output coordinates. The following theorem shows that such couplings are in fact incompatible with the Shapley axioms, yielding a rigidity phenomenon.

\begin{theorem}[Rigidity of vector-valued Shapley operators]\label{thm:coordinatewise_decomposition}

Let $\Phi:\MC G^m_{[n]}\to \mathbb R^{mn}$ be a vector-valued Shapley operator satisfying the axioms in Definition \ref{def:SHAP_operator_vector}. Then, for every $k\in[m]$, the $k$-th coordinate of $\Phi(v)$ depends only on the $k$-th coordinate of $v$; equivalently, there exists a scalar value operator $\phi:\MC G_{[n]}\to\mathbb{R}^n$ such that for all $v\in \MC G^m_{[n]}$ and all $i\in[n]$,
\begin{align*}
    \Phi_i(v)=\big(\phi_i(\pi_1\circ v),\dots,\phi_i(\pi_m\circ v)\big),
\end{align*}
where for all $k\in[m]$, $\pi_k:\mathbb{R}^m\to\mathbb{R}$ denotes the coordinate projections in $\mathbb{R}^m$.
\end{theorem}

\begin{proof}[Proof sketch.]
    First observe that every $v\in \mathcal{G}_{[n]}^m$ can be uniquely decomposed into its scalar coordinate components $v^k$, $k\in [m]$. Since $\Phi$ is additive, it suffices to understand how the operator acts on games that are supported on a single coordinate.

    Fix $k\in [m]$ and consider a game $v$ whose values lie entirely in the $k$-th coordinate. By the dummy axiom, players whose marginal contribution vanishes receive the zero vector. By symmetry, players with identical marginal contributions must receive identical payoff vectors. Finally, vector efficiency requires that the sum of all payoff vectors equals the grand-coalition value $v([n])$, which lies entirely in the $k$-th coordinate.

    These three properties together force each $\Phi_i(v)$ to lie in the same coordinate direction as $v$. In particular, no component in the remaining $m-1$ coordinates can appear. Thus, the action of $\Phi$ on coordinate-supported games is necessarily coordinate-preserving.

    Since any $v\in \mathcal{G}_{[n]}^m$ can be written as a sum of coordinate-supported games and $\Phi$ is additive, the absence of cross-coordinate effects extends to arbitrary games. Consequently, there exists a scalar value operator such that each coordinate of $\Phi(v)$ depends only on the corresponding scalar component $v^{(k)}$. This establishes the component-wise decomposition. Complete details are provided in Appendix~\ref{appendix:proofs:main:results}.
\end{proof}


The decomposition result contained in Theorem \ref{thm:coordinatewise_decomposition} immediately yields a full characterization of the vector-valued Shapley operator.

\begin{theorem}\label{thm:existence:uniqueness:Shapley:formula}
There exists a unique vector- valued operator $\Phi: \MC G^m_{[n]}\to \R^{mn}$ satisfying efficiency, symmetry, dummy, and additivity as in Definition \ref{def:SHAP_operator_vector}. Moreover, for every game $v\in\MC G^m_{[n]}$ and every player $i\in[n]$, this operator is given by the permutation formula
\begin{align}\label{def:Phi:i:vector}
     \phi_i(v)=\frac{1}{n!} \sum_{\pi \in \mathfrak{S}_{[n]}} \Big[v\big(P_i(\pi) \cup \{i\}\big) - v\big(P_i(\pi)\big) \Big]\in \mathbb{R}^m,
\end{align}
or, equivalently, by
\begin{align}\label{def:Phi:i:vector_combinatory}
     \phi_i(v)=\sum_{S\subseteq ([n]\setminus \{i\})} \dfrac{|S|! (n-|S|-1)!}{n!} \Big[v(S\cup \{i\}) - v(S) \Big]\in \R^m.
\end{align}
\end{theorem}

\begin{proof}[Proof sketch.]
Existence is obtained by defining $\Phi$ component-wise: we apply the classical scalar Shapley operator to each coordinate game $\pi_k \circ v$ and stack the resulting allocations into a vector in $\mathbb{R}^m$. The Shapley axioms in Definition \ref{def:SHAP_operator} then follow coordinate-wise from the scalar axioms. Uniqueness follows from Theorem \ref{thm:coordinatewise_decomposition} together with the scalar uniqueness Theorems \ref{thm:SHAP_scalar} or \ref{thm:1dim:Shapley:values} applied to each coordinate. Complete details are provided once again in Appendix \ref{appendix:proofs:main:results}. 
\end{proof}

Theorem \ref{thm:existence:uniqueness:Shapley:formula} shows that extending Shapley values to vector-valued games preserves the axiomatic structure of the scalar theory. In practice, SHAP values for multi-output models are therefore obtained by applying the standard Shapley formula independently to each output coordinate.

\begin{corollary}[Impossibility of joint-output Shapley attribution]\label{cor:immpossible}
There exists no feature-attribution method for vector-valued predictors that simultaneously
\begin{itemize}
    \item[(i)] satisfies the Shapley axioms of efficiency, symmetry, dummy player, and additivity;
    \item[(ii)] assigns non-component-wise (joint) attributions across output coordinates.
\end{itemize}
In particular, any attribution rule that explicitly couples different outputs must violate at least one of the Shapley axioms.
\end{corollary}

While Theorems \ref{thm:coordinatewise_decomposition} and \eqref{thm:existence:uniqueness:Shapley:formula} establish a structural rigidity and uniqueness result for vector-valued Shapley operators, they do not address the quantitative behavior of the attribution rule. In practical learning scenarios, predictors are subject to perturbations arising from finite-sample effects, architectural choices, and optimization variability. It is therefore natural to ask whether the Shapley operator is stable under perturbations of the underlying cooperative game. 

\medskip 
\noindent Given $a=(a_1,\ldots,a_n)\in (\mathbb{R}^{m})^{n}$, denote
\begin{align*}
    \|a\|_{\mathcal{A},\infty}:= \max_{i\in [n]} \|a_i\|_{\infty}.
\end{align*}
Moreover, for any game $v\in \mathcal{G}_{[n]}^m$, define the norm 
\begin{align*}
    \|v\|_{\mathcal{G},\infty}:= \max_{S\subseteq [n]} \|v(S)\|_{\infty}
\end{align*}
and the marginal seminorm
\begin{align*}
    \|v\|_{\Delta,\infty}:=\max_{i\in [n]} \max_{S\subseteq [n]\setminus \{i\}} \|v(S\cup \{i\}) - v(S)\|_{\infty}.
\end{align*}

The following result establishes a Lipschitz-type continuity property of the vector-valued SHAP operator with respect to both the marginal seminorm and the sup norm on games.
\begin{proposition}
    \label{prop:estimates:Phi}
    Let $\Phi:\mathcal{G}_{[n]}^m \to (\mathbb{R}^{m})^n$ be the vector-valued SHAP operator given by Theorem \ref{thm:existence:uniqueness:Shapley:formula}. Then,
    \begin{itemize}
        \item[(i)] $\Phi$ is a linear operator.
        \item[(ii)] for all $u,v\in \mathcal{G}_{[n]}^m$, we have 
    \begin{align*}
        \|\Phi(u)- \Phi(v)\|_{\mathcal{A},\infty}\leq  \|u-v\|_{\Delta,\infty}. 
    \end{align*}
        \item[(iii)] For all $u,v\in \mathcal{G}_{[n]}^m$, we have 
    \begin{align*}
        \|\Phi(u)- \Phi(v)\|_{\mathcal{A},\infty}\leq 2 \|u-v\|_{\mathcal{G},\infty}. 
    \end{align*}
    \end{itemize}    
\end{proposition}

\begin{proof}
    Clearly, (i) follows from \eqref{def:Phi:i:vector}. Now, let $g:=u-v\in \mathcal{G}_{[n]}^m$. Fix $i\in [n]$. Then, by triangle inequality we have
    \begin{align}
        \label{ineq:SHAP}
        \|\phi_i(g)\|_{\infty}=\left\| \sum_{S\subseteq [n]\setminus \{i\}} w(S) (g(S\cup \{i\}) - g(S)) \right\|_{\infty} \leq \max_{S\subseteq \{[n] \setminus \{i\} \}} \|g(S\cup \{i\}) - g(S)\|_{\infty},
    \end{align}
    where we have used the fact that $\sum_{S\subseteq [n]\setminus \{i\}} w(S)=1$. Then, taking the maximum on $i\in [n]$ in \eqref{ineq:SHAP}, we deduce that
    \begin{align*}
        \|\Phi(g)\|_{\mathcal{A},\infty}\leq \|g\|_{\Delta,\infty},
    \end{align*}
    which proves (ii). Finally, (iii) follows from \eqref{ineq:SHAP}, and the fact that 
    \begin{align*}
        \|g(S\cup \{i\}) - g(S)\|_{\infty}\leq \|g(S\cup \{i\})\|_\infty +\|g(S)\|_\infty \leq 2 \|g\|_{\mathcal{G},\infty}.
    \end{align*}
    and taking the maximum on $i\in [n]$ in the resulting inequality.
\end{proof}

\subsection{On axiomatic choices in the vector-valued setting}\label{subsec:SHAP_discussion}

The extension of the Shapley axioms to vector-valued payoffs deserves explicit justification. In particular, efficiency and additivity are imposed directly in the output space $\mathbb{R}^m$, rather than after scalarization. This choice reflects the interpretation of feature attribution as a decomposition of the realized prediction vector itself, rather than of a derived scalar utility. Under this interpretation, vector-valued efficiency enforces exact reconstruction of the prediction deviation, while additivity ensures linear consistency across independent predictive contributions. Relaxing either axiom enlarges the space of admissible attribution rules, but at the cost of abandoning the classical Shapley characterizations. 

Cooperative games with vector-valued or multicriteria payoffs have been studied in the game-theoretic literature under a variety of frameworks, motivated by the absence of a canonical total order in $\mathbb R^m$. In these settings, solution concepts typically rely on additional structure, such as preference relations, ordering cones, or scalarization procedures, in order to compare vector payoffs and define notions of fairness or optimality; see, for instance, \cite{patrone2007multicriteria} for a multicriteria cooperative-game perspective. In contrast, the approach adopted here deliberately avoids introducing any preference or scalarization structure on the output space. Instead, we extend the classical Shapley axioms verbatim to vector-valued payoffs, interpreting efficiency and additivity directly in $\mathbb R^m$. 

The rigidity result established in Theorem \ref{thm:coordinatewise_decomposition} depends crucially on this axiomatic framework. This choice is not ad hoc. The axioms used in Definition \ref{def:SHAP_operator_vector} represent the most direct extension of the original Shapley principles to the multi-output setting, preserving their normative interpretation. Efficiency requires that the total vector-valued payoff of the grand coalition be fully distributed among the players, without loss or creation of value. Symmetry enforces identical treatment of players with identical marginal contributions. The dummy player axiom excludes spurious attributions to features with no effect, and additivity expresses linear consistency with respect to the superposition of independent games. Together, these principles form the axiomatic backbone underlying SHAP-based interpretability in the scalar-output case.

Within this framework, additivity plays a particularly structural role. By enforcing linearity of the attribution operator at the vector level, it enables the decomposition arguments used in the proof of Theorem \ref{thm:coordinatewise_decomposition} and ultimately forces output-wise separation. The resulting no-go theorem should therefore be understood as a conditional statement: output-coupled attributions are incompatible with the Shapley axioms when these are imposed directly on vector-valued payoffs.

At the same time, other axiomatic choices are conceivable in multi-output or multi-objective settings. For instance, one may replace vector-valued efficiency with scalarized or weighted notions of efficiency, relax additivity, or introduce preference structures over outputs. Such approaches may allow attribution rules that explicitly mix output coordinates. However, they necessarily depart from the classical Shapley framework and do not retain its axiomatic characterization in the scalar case.

From this perspective, Theorem \ref{thm:coordinatewise_decomposition} precisely delineates the scope of Shapley-consistent interpretability for multi-output predictors. It does not claim that joint-output explanations are impossible in general, but rather that they cannot coexist with the classical Shapley axioms when these are extended verbatim to vector-valued games. This clarification makes explicit which assumptions are responsible for the rigidity result and identifies the axioms that would need to be relaxed in order to move beyond component-wise attribution.

\subsection{SHAP values for ML models}\label{subsec:SHAP_ML}

The abstract framework developed in Sections \ref{subsec:SHAP_scalar} and \ref{subsec:SHAP_vector} characterizes fair allocation rules for cooperative games independently of any specific application. We now explain how this formalism connects to SHAP explanations in ML and how multi-output predictors naturally fit into the vector-valued setting.

Consider a ML model $f:\R^n \to \R^m$ and let $X = (X_1, \ldots, X_n)\in\R^n$ be a random input vector distributed according to a background probability measure $\mu$ on $\R^n$. For a given observation $x = (x_1, \ldots, x_n)\in\R^n$, we associate to $f$ the \textit{centered characteristic function}
\begin{align}\label{def:v:f:R:m}
    v_f(S) = \mathbb{E}_{X\sim\mu} \Big[f(X)\big|X_S=x_S\Big]- \mathbb{E}_{X\sim \mu}[f(X)] \in\R^m, \qquad S\subseteq[n].    
\end{align}

Here, $X_S\in\R^{|S|}$ denotes the subvector of features indexed by $S$ and where the conditional expectation is understood in the sense of regular conditional probabilities. Moreover, we notice that, by construction, $v_f(\varnothing)=0$. Hence, $v_f\in \mathcal{G}_{[n]}^m$.

The quantity $v_f(S)$ represents the expected output of the model when the features in $S$ are fixed at their observed values $x_S$, while all remaining features vary according to their background distribution. In other words, it expresses what the model ``expects to predict'' once the information in the subset $S$ is known. The incremental difference $v_f(S \cup\{i\}) - v_f(S)$ then quantifies how much additional predictive power the feature $i\in[n]$ adds beyond the information already provided by $S$. 

In this context, the SHAP value of the $i$-th feature is defined as the average of these marginal contributions to the model output, computed over all possible orderings in which features could enter the coalition: 
\begin{align}\label{eq:phi_ML}
    \phi_i(f;x) = \frac{1}{n!}\sum_{\pi\in\mathfrak{S}_{[n]}} \Big[v_f\big(P_i(\pi)\cup\{i\}\big)-v_f\big(P_i(\pi)\big)\Big] \in\R^m.    
\end{align}
Equivalently, this can be expressed in the \textit{subset form}
\begin{align}\label{formula:Shapley:values:mdim}
    \phi_i(f;x)=\sum_{S\subseteq ([n]\setminus \{i\})} \dfrac{|S|! (n-|S|-1)!}{n!} \Big(v_f(S\cup \{i\}) - v_f(S) \Big)\in \R^m, 
\end{align}
where each subset $S$ contributes with a weight representing the proportion of possible orderings in which the subset $S$ appears before feature $i$. 

This construction makes the correspondence between cooperative games and model interpretability fully explicit: the function $v_f$ plays the role of the characteristic function, the features are the players, and the SHAP values $\phi_i(f;x)$ quantify each feature's expected marginal contribution to the prediction $f(x)$ relative to the baseline $\mathbb{E}_{X\sim\mu}[f(X)]$. The following result, whose complete proof will be given once again in Appendix \ref{appendix:proofs:main:results}, formalizes this correspondence, showing that the sum of all feature attributions exactly reconstructs the model's output deviation from its expected value.

\begin{proposition}\label{corollary:equality:efficiency}
For all $f:\R^n \to \R^m$ and all input data $x\in \R^n$, it holds
\begin{align}\label{equality:efficiency}
    f(x)-\mathbb{E}_{X\sim\mu}[f(X)]= \sum_{i=1}^n \phi_i(f;x),
\end{align}
with $\phi_i(f;x)$ defined in \eqref{eq:phi_ML}.
\end{proposition}

Moreover, the stability result established in Proposition \ref{prop:estimates:Phi} can be naturally translated into the ML setting through the following proposition, whose proof is immediate and left to the reader. 
\begin{proposition}
    \label{prop:estimates:Phi:02}
    Let $f,g:\mathbb{R}^n \to \mathbb{R}^m$ be measurable predictors and let $v_f,v_g\in \mathcal{G}_{[n]}^m$ be defined by \eqref{def:v:f:R:m} with respect to the same $\mu$ and $x$. For $h:\mathbb{R}^n \to \mathbb{R}^m$, we define 
    \begin{align*}
        \Phi(h;x)=(\phi_1(h;x),\ldots, \phi_n(h;x)),
    \end{align*}
    where $\phi_i$ is given by \eqref{formula:Shapley:values:mdim}, for each $i\in [n]$. Then, $\Phi$ is a linear operator and 
    \begin{align}
        \label{SHAP:stability:bound}
        \|\Phi(f;x) - \Phi(g;x)\|_{\mathcal{A},\infty}\leq \|v_f-v_g\|_{\Delta,\infty}.
    \end{align}
    Moreover, if $\|f-g\|_{L^\infty(\mu)}<\infty$, then
    \begin{align*}
        \|\Phi(f;x) - \Phi(g;x)\|_{\mathcal{A},\infty}\leq 2 \|f-g\|_{L^\infty(\mu)}.
    \end{align*}
\end{proposition}

\section{Numerical experiments}
\label{section:Numerical:experiments}

In addition to illustrating the structural consequences of Theorem \ref{thm:coordinatewise_decomposition}, the experiments are also informed by the quantitative stability estimate established in Proposition \ref{prop:estimates:Phi:02}. In particular, this result implies that if two predictors $f$ and $g$ are such that their characteristic functions $v_f$ and $v_g$ are close in the marginal seminorm, then their SHAP explanations must also be close in the allocation norm. Consequently, strong agreement between multi-output and single-output SHAP values is expected whenever the corresponding predictors remain uniformly close on the data distribution.

We notice that the theoretical analysis in Section \ref{sec:SHAP_math} is formulated entirely in terms of the exact Shapley operator defined by formulas \eqref{eq:phi_ML} and \eqref{formula:Shapley:values:mdim}, independently of any specific computational approximation. In practice, exact computation of SHAP values is intractable for high-dimensional nonlinear models, and one must use approximation schemes. In this work, we employ DeepExplainer \cite{lundberg2017unified}, which provides a tractable and widely-used approximation of Shapley values for neural networks. Importantly, the axiomatic rigidity and stability results established in Section \ref{sec:SHAP_math} apply to the exact Shapley operator and are therefore independent of the particular explainer used in the experiments. DeepExplainer is used solely as a computational surrogate to approximate these theoretical quantities.

\subsubsection*{Reproducibility details} The simulations were conducted on a laptop running Windows 11 (64-bit, build 26100). The system is equipped with an Intel Core Ultra 5.225H (14 cores, base 1.7 GHz, x86 64/AMD64) and 32 Gb of RAM. The algorithms were implemented in Python 3.13.5 (Anaconda). The code and the dataset used to reproduce the example are available on GitHub \cite{Github}.

\subsection{Experimental setup and dataset description}

The numerical experiments are conducted on a biomedical tabular dataset comprising metabolic, biochemical, and clinical variables measured at the individual level. The dataset consists of 15,931 samples with complete observations and a total of 75 numerical input variables, mainly corresponding to blood-based metabolites, lipid markers, inflammatory indicators, and standard clinical measurements.


The predictive task is formulated as a multi-output learning problem with three heterogeneous outputs: chronological age, MetSCORE, and biological sex. MetSCORE \cite{gil2024metscore} is a continuous metabolic risk index derived from metabolomic profiles, designed to capture systemic metabolic alterations through a weighted combination of blood-based biomarkers. Sex is encoded as a binary variable and treated as a classification target, while age and MetSCORE are addressed as regression outputs.

Prior to model training, a systematic exploratory analysis of the input space is carried out. Pairwise Pearson correlations among input variables reveal the presence of several highly correlated feature groups. To mitigate redundancy and multicollinearity effects, an automatic correlation-based pruning procedure is applied, removing input variables with absolute pairwise correlation exceeding a fixed threshold while favoring features with higher relevance of the outputs. This results in a reduced yet informative subset of input variables used throughout the learning and interpretability analyses.

In addition, principal component analysis (PCA) is performed on the standardized input data to assess the intrinsic dimensional structure of the dataset. The cumulative explained variance curve (depicted in Figure \ref{f:PCA}) exhibits a smooth concave curve, indicating a moderately high intrinsic dimensionality.

\begin{figure}[!h]
    \centering
    \includegraphics[width=0.5\linewidth]{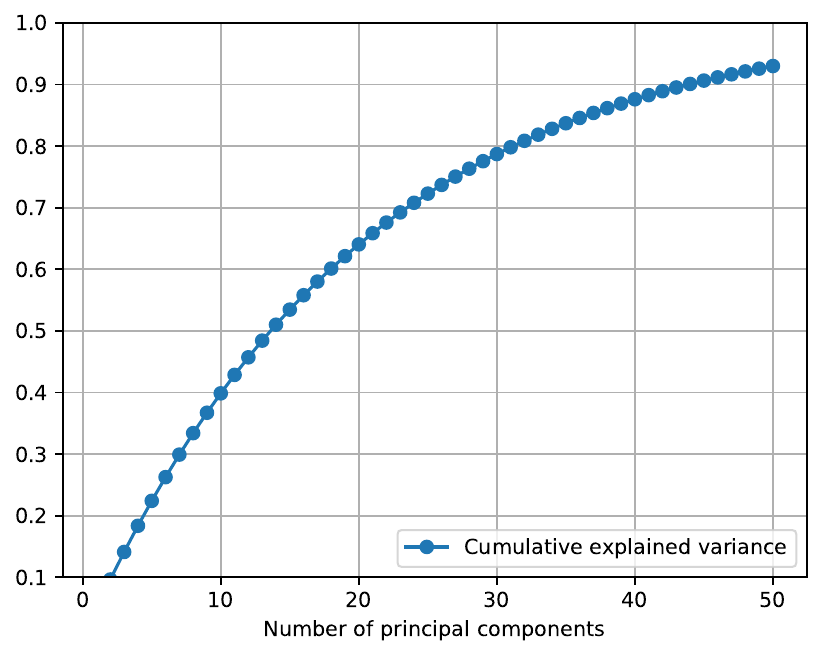}
    \caption{Cumulative explained variance obtained from PCA on the standardized input features. The shape of the curve indicates that the dataset does not admit a strong low-dimensional linear representation.}
    \label{f:PCA}
\end{figure}



In particular, approximately $60\%$ of the variance is captured by the first 20 principal components, while about $90\%$ requires on the order of 45-50 components. This behavior suggests that, although correlations are present, the dataset retains a rich and distributed informational structure, which motivates the use of flexible nonlinear models and supports the relevance of feature-level interpretability analyses.

\subsection{Model architectures and training}

We consider two multi-output neural network architectures, denoted by M1 and M2, which share the same output structure but differ in the expressiveness of their internal representation.

Model M1 is a linear multi-output baseline composed of two heads: one jointly predicting the continuous output (age and MetSCORE) and one predicting sex via a linear classification logit. All outputs are linear functions of the input features.

Model M2 extends this baseline by introducing a shared nonlinear representation implemented as a multilayer perceptron (MLP) with ReLU activations. The same two-head output structure is retained, with linear heads operating on the shared latent representation.

For comparison purposes, we also consider single-output models associated with each target variable (i.e. age, MetSCORE, and sex). These models share the same architectural design as the multi-output model M2, but are trained independently for each output.

\subsubsection{Training protocol}
All input variables were standarized using the empirical mean and standard deviation computed on the training set. The dataset was split $70\%$ training, $15\%$ validation, and $15\%$ test sets using a fixed random seed.

Models were trained using the Adam optimizer with learning rate $10^{-3}$, batch size $128$, and a maximum of 200 epochs. Early stopping with patience 20 was applied based on validation loss, and the best-performing model was retained.

For multi-output models, the loss function was the sum of task-specific losses:
\begin{align*}
    \mathcal{L}:=\mathcal{L}^{\text{age}} + \mathcal{L}^{\text{MetSCORE}} + \mathcal{L}^{\text{Sex}},
\end{align*}
with equal weighting across tasks. Age and MetSCORE were trained using mean squared error (MSE) loss, while Sex was trained using binary cross-entropy applied to the logit output. On the other hand, single-output models were trained independently using the corresponding loss. SHAP values were approximated using DeepExplainer with 100 background samples drawn from the training set.

\subsection{Training and evaluation of the models}

We first analyze the predictive performance of the multi-output models M1 and M2, using the linear architecture M1 as a baseline and comparing it with its nonlinear counterpart M2. Quantitative results for the continuous outputs (Age and MetSCORE) are reported in the Table \ref{table:age:metscore:all:models}, while classification metrics for sex prediction are summarized separately in the Table \ref{table:sex:all:models}.  

\begin{table}[!h]
    \centering
    \begin{tabular}{|c||c||c|c|c|} \hline
         Model & Output & $R^2$ & RMSE & MAE \\ \hline \hline 
         \multirow{2}{*}{M1 Multi-output} & Age & 0.5614 & 8.9306 & 6.9946 \\ \cline{2-5}
         & MetSCORE & 0.728266 & 0.1343 & 0.0963 \\ \cline{2-5}
         \hline\hline
         \multirow{2}{*}{M2 Multi-output} & Age & 0.6206 & 8.3061 & 6.3454 \\ \cline{2-5}
         & MetSCORE & 0.8123 & 0.1116 & 0.0764 \\ \cline{2-5}
         \hline\hline
         \multirow{2}{*}{M2 Single-output} & Age & 0.6529 & 7.9438 & 6.0959 \\ \cline{2-5}
         & MetSCORE & 0.8173 & 0.1101 & 0.0745 \\ \cline{2-5}
         \hline 
    \end{tabular}
    \vspace{6pt}
    \caption{Test-set regression performance for the continuous outputs (Age and MetSCORE). Results are reported for the linear multi-output baseline (M1), the nonlinear multi-output model (M2), and the corresponding single-output models based on the same architecture.} \label{table:age:metscore:all:models}
\end{table}

For the Age prediction task, the nonlinear multi-output model M2 clearly outperforms the linear baseline M1. The coefficient of determination increases from $R^2=0.5614$ to $R^2=0.6206$, together with a consistent reduction in both RMSE (from 8.93 to 8.31) and MAE (from 6.99 to 6.35). The improvement highlights the relevance of nonlinear feature interactions when modeling age-related variability from metabolic and clinical inputs.

The multi-output model's lower accuracy in predicting chronological age suggests that the shared feature space prioritizes metabolic and sex-specific signals over mere temporal progression. For a metabolic clock, the primary objective is not to maximize chronological precision, but to accurately reflect the individual's metabolic age. By accounting for MetSCORE and sex within the same architecture, the model effectively filters out chronological noise, yielding an estimation that represents biological status rather than calendar years. This trade-off may be desirable, as it aligns the prediction more closely with physiological markers than purely chronological progression.

A similar, and even more pronounced, improvement is observed for MetSCORE prediction. The multi-output M2 model achieves a substantially higher $R^2$ value (0.8123) compared to M1 (0.7283), accompanied by lower RMSE and MAE values. This suggests that MetSCORE benefits particularly from the increased expressiveness of the shared nonlinear representation.

\begin{table}[!h]
    \centering
    \begin{tabular}{|c||c|c|c|} \hline 
         Model & ACC & AUC & F1-score \\ \hline \hline  
         M1 Multi-output & 0.9373 & 0.9814 & 0.9462 \\ \hline \hline
         M2 Multi-output & 0.9485 & 0.9869 & 0.9559 \\ \hline \hline 
         M2 Single-output & 0.9498 & 0.9849 & 0.9568 \\ \hline 
    \end{tabular}
    \vspace{6pt}
    \caption{Test-set classification performance for sex prediction. Results are shown for the multi-output models M1 and M2, as well as for the single-output model derived from the M2 architecture.}
    \label{table:sex:all:models}
\end{table}

Regarding sex classification in the Table \ref{table:sex:all:models}, both multi-output models achieve strong performance, reflecting the presence of highly informative features for this task. Nevertheless, M2 consistently improves upon M1, with accuracy increasing from 0.9373 to 0.9485 and AUC from 0.9814 to 0.9869. These results indicate that the nonlinear shared representation does not introduce adverse interference between tasks and remains effective for both regression and classification objectives.

We now compare the multi-output M2 model with its single-output counterparts. For Age, the single-output model achieves a higher predictive accuracy than the multi-output version, with $R^2=0.6529$ versus $R^2=0.6206$, and further reductions in RMSE, and MAE. This behavior is expected, as the single-output model can fully specialize its representation to the age prediction task.

For MetSCORE, the difference between single-output and multi-output M2 models is markedly smaller. While the single-output model slightly improves all regression metrics, the gains are marginal, indicating that MetSCORE prediction is largely compatible with joint learning alongside the other outputs.

Finally, for sex classification, the performance of the multi-output and single-output M2 models is nearly indistinguishable. Accuracy, AUC, and F1-score differ only marginally, suggesting that sex prediction is robust to parameter sharing and does not require task-specific specialization.

We analyze the computational cost of the proposed models in terms of training time and RAM. The Table \ref{table:time:and:RAM} reports the total training time and the peak RAM usage for the M2 architecture, considering both the multi-output formulation and the single-output counterparts.

\begin{table}[!h]
    \centering
    \begin{tabular}{|c||c|c|} \hline 
         & Time (s) & RAM (MB) \\ \hline 
         Age & 81.71 & 605.96 \\ \hline
         MetSCORE & 84.60 & 603.93 \\ \hline 
         Sex & 83.13 & 603.11 \\ \hline 
         Total M2 single-output & 249.44  & 605.96 \\ \hline 
         Total M2 multi-output & 89.3678 & 599.15 \\ \hline 
    \end{tabular}
    \vspace{6pt}
    \caption{Training time and peak RAM usage for the M2 architecture in the multi-output and single-output settings. Single-output models are trained independently for each target (Age, MetSCORE, and Sex), and their total training time is reported as the sum over tasks, while the RAM value corresponds to the maximum peak memory observed among the single-output runs.}

    \label{table:time:and:RAM}
\end{table}

For the single-output setting, three independent models are trained, one for each target variable (Age, MetSCORE, and Sex). Each of these models requires a comparable amount of memory, with RAM usage slightly above 600 MB. The total training time for the single-output approach is therefore cumulative, amounting to approximately 249 seconds. In contrast, the reported RAM usage for the total M2 single-output corresponds to the maximum peak memory observed among the three single-output models, since they are trained sequentially rather than simultaneously.

The multi-output M2 model exhibits a markedly different behavior. Despite predicting all three outputs jointly, its total training time is approximately 89 seconds, which is significantly lower than the cumulative time required by the single-output models. In terms of RAM, the multi-output model uses slightly less RAM (about 599 MB) than any of the single-output models, indicating that parameter sharing does not introduce additional memory overhead. 

\subsection{SHAP values}

The local attributions of SHAP values provided by beeswarms for both multi-output and single-output models are depicted in the Figure \ref{fig:beeswarm:age:metscore:sex}:

\begin{figure}[!h]
    \centering
    \includegraphics[width=1.02\linewidth]{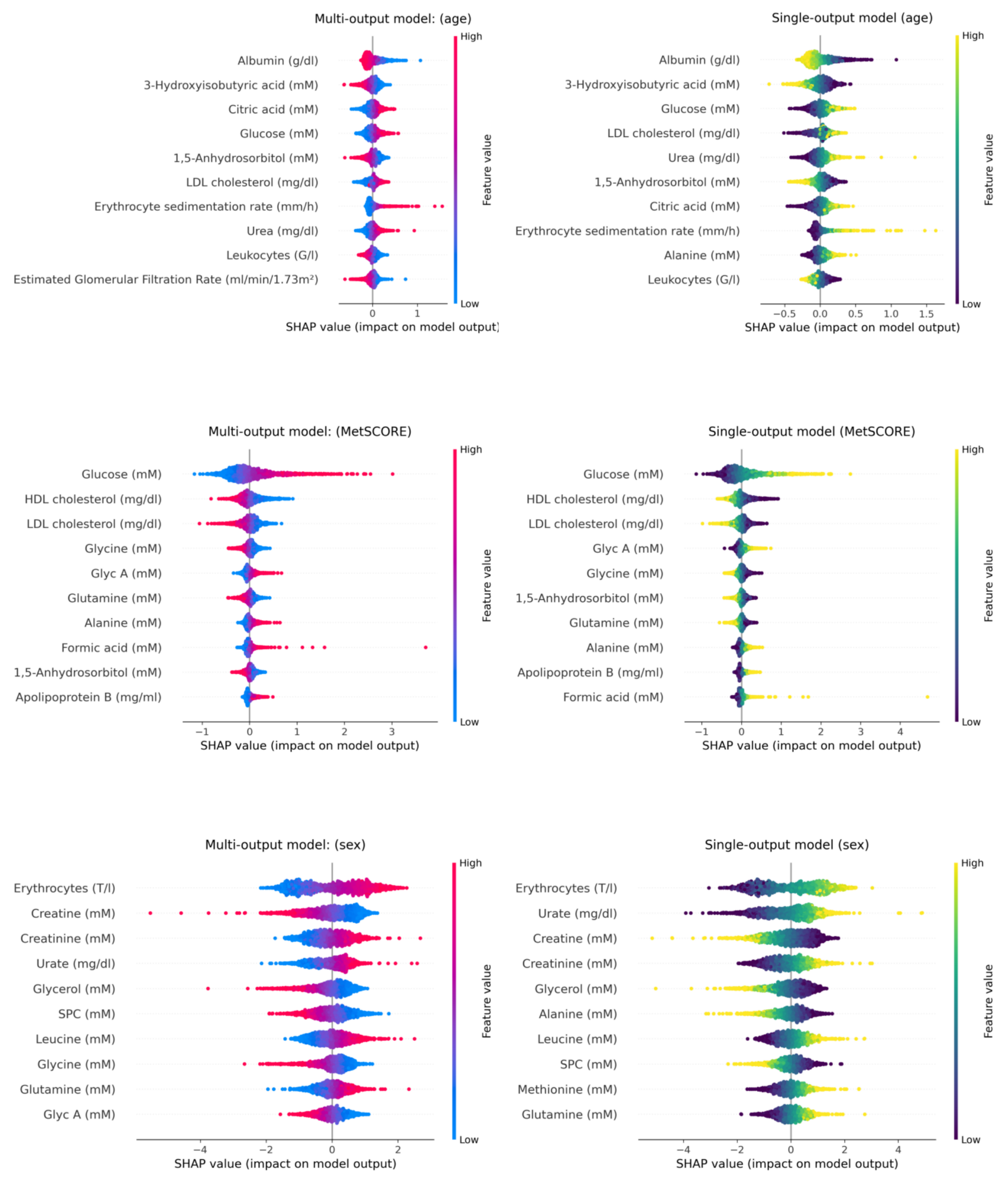}
    \caption{SHAP beeswarm plots comparing multi-output and single-output models for the three prediction tasks considered. From top to bottom: age, MetSCORE, and sex. In each row, the left panel corresponds to the multi-output model, while the right panel shows the associated single-output model trained independently for the same target.}
    \label{fig:beeswarm:age:metscore:sex}
\end{figure}

It should be noted that the apparent discrepancy between visual dispersion and feature ranking arises because SHAP summary plots are ordered according to the mean absolute contribution across the dataset, rather than by the variance or the maximum magnitude of the SHAP values. As a consequence, features with rare but extreme effects may appear visually more dispersed, while having a smaller average impact and therefore a lower rank.

For age prediction, the most influential variables (such as albumin, glucose-related metabolites, lipid markers, and renal function indicators) exhibit consistent monotonic trends across both models. High or low feature values induce SHAP contributions with the same sign in the multi-output and single-output settings, suggesting that the shared representation does not distort the underlying physiological associations. Minor differences are mainly observed in the dispersion of SHAP values, with the multi-output model typically yielding slightly more concentrated distributions, which can be interpreted as a regularization effect induced by parameter sharing across tasks.

In the case of MetSCORE, the agreement between the two approaches remains strong for the top-ranked features, particularly glucose, cholesterol-related variables, and selected amino acids. The multi-output model reproduces the same ordering of feature importance and preserves the directionality of contributions. However, for some metabolites, the single-output model displays heavier tails in the SHAP distributions, indicating a higher sensitivity to extreme feature values. This behavior suggests that multi-task coupling may dampen extreme attributions by leveraging correlated information from the other outputs.

For the sex classification task, the beeswarm plots reveal clearer separations in SHAP values for hematological and metabolic features. Again, the sign and relative magnitude of the dominant contributions are consistent between multi-output and single-output models. The multi-output explanations appear slightly smoother, while the single-output model exhibits sharper contrasts for certain variables, which is consistent with the higher flexibility of a task-specific model.

The SHAP values reveal a consistent biological logic across both architectures. For age, the model identifies a multi-systemic signature: chronic inflammation is represented by albumin and ESR, while the presence of 3-hydroxyisobutyrate (a valine catabolite linked to insulin resistance) alongside glucose and anhydrosorbitol reflects the metabolic shifts inherent to aging. Furthermore, age-related changes in organ function and lipid metabolism are captured via urea and LDL levels. Regarding MetSCORE, the dominance of glucose, followed by LDL and HDL, aligns with the clinical definition of metabolic syndrome, where hyperglycemia and dyslipidemia are primary drivers. For sex, the models rely on well-known physiological dimorphisms, such as erythrocyte counts and creatinine (reflecting muscle mass), alongside urate and lipid species like SPC and glycerol, which vary according to sex-specific adipose tissue distribution and hormonal profiles.

The Table \ref{num:exp:table:2:cosine} reports the cosine similarity and Spearman rank correlation between feature-importance vectors computed for each output variable. Both metrics assess the agreement between the relative importance of input features, with cosine similarity capturing global alignment and Spearman correlation quantifying consistency in feature ranking.

\begin{table}[!h]
    \centering
    \begin{tabular}{|c||c|c|} \hline 
        Output & Cosine similarity & Spearman correlation \\ \hline \hline  
         Age & 0.9867 & 0.9460 \\ \hline 
         MetSCORE & 0.9949 & 0.9003\\ \hline 
         Sex & 0.9740 & 0.8639\\ \hline 
    \end{tabular}
    \vspace{6pt}
    \caption{Cosine similarity and Spearman rank correlation between feature-importance ranking, computed over the full set of input features for each output variable.}
    \label{num:exp:table:2:cosine}
\end{table}

The high cosine similarity and Spearman rank correlation observed across all outputs can be interpreted in light of Proposition \ref{prop:estimates:Phi:02}. Indeed, although the internal representations of the multi-output and single-output models differ, their induced characteristic functions remain sufficiently close on dominant coalitions. The stability bound \ref{SHAP:stability:bound} ensures that small discrepancies in conditional expectations cannot produce large deviations in feature attributions. The near-perfect alignment of feature-importance vectors therefore reflects the quantitative robustness of the Shapley operator rather than a coincidental empirical agreement.


The Spearman rank correlations, while slightly lower, remain consistently high, ranging from 0.86 to 0.95. Age shows the strongest rank agreement (0.9460), indicating that not only the magnitude but also the ordering of the most influential features is highly consistent. For MetSCORE and Sex, the somewhat lower Spearman values reflect mild reordering among less influential features, while still preserving a strong agreement among the most relevant ones.

\section{Conclusions and open problems}\label{sec:conclusion}

In this work, we studied the interpretability of multi-output ML models through the lens of Shapley-based feature attribution. From a theoretical perspective, we formulated feature attribution for vector-valued predictors within an axiomatic framework extending the classical Shapley principles. Our main result establishes a rigidity property: any attribution method satisfying efficiency, symmetry, dummy player, and additivity necessarily decomposes component-wise across output coordinates. This shows that Shapley values for multi-output models cannot explicitly couple different outputs, providing a formal justification for the widespread practice of computing SHAP values independently for each target. From a mathematical perspective, our results show that component-wise SHAP explanations in multi-output models are not a heuristic design choice but the unique consequence of enforcing classical Shapley fairness in vector-valued prediction problems.

Several limitations of the present work should be explicitly pointed out. First, our results do not provide a new attribution method, but rather a structural characterization of existing ones. Second, the rigidity theorem applies only to attribution rules satisfying the classical Shapley axioms in their full strength; alternative notions of fairness may yield different conclusions. Finally, while our numerical experiments focus on biomedical data, the theoretical framework is agnostic to the application domain and does not encode domain-specific semantics.

Building on this theoretical foundation, we investigated the practical implications of multi-output modeling for interpretability. Numerical experiments on a representative biomedical example showed that SHAP explanations produced by multi-output and single-output nonlinear models are quantitatively consistent. In particular, cosine similarity and Spearman correlation analyses demonstrate an almost perfect agreement for the most influential features, with minor discrepancies confined to low-importance variables. These results confirm that joint learning preserves the interpretability of each output while remaining fully consistent with the axiomatic structure of Shapley-based explanations. In this sense, multi-output learning can exploit shared structure across outputs for prediction and efficiency. However, interpretability within the Shapley paradigm remains inherently output-separable, as attribution cannot reflect inter-output dependencies without abandoning classical fairness axioms. We also analyzed the computational consequences of multi-output learning.



Furthermore, the results presented in this work also suggest several directions for future research:
\begin{itemize}
    \item[1.] \textbf{Beyond Shapley axioms for multi}-\textbf{output attribution.} Although we have established that the Shapley operator decomposes component-wise for vector-valued predictors. It remains open to determine whether alternative formulations could capture joint dependencies or interactions between outputs. Developing such a theory would deepen our understanding of interpretability in structured and multi-task prediction problems.
    \item[2.] \textbf{Exploiting output correlations in SHAP estimation.} Since multi-output predictors are particularly effective when outputs are strongly correlated, an important open question is whether correlation structure can be explicitly integrated into the Shapley value estimation process. This may lead to improved sampling strategies, variance-reduction techniques, or new approximation schemes tailored to the multi-output setting. 
    \item[3.] \textbf{Multi}-\textbf{objective optimization and Pareto structure in interpretability.} Since multi-output training naturally corresponds to a multi-objective optimization problem, it would be valuable to explore whether Pareto-optimality principles can inform or constrain Shapley-based explanations. Understanding the relationship between Pareto structure and feature attribution may lead to new theoretical insights or alternative interpretability frameworks.
    \item[4.] \textbf{Dynamic and temporal models.} Another promising direction involves extending the SHAP framework to time-dependent predictors such as recurrent neural networks. In such systems, feature contributions evolve over time, and the notion of fairness may require integrating marginal effects along temporal trajectories. Defining ``dynamic SHAP values'' consistent with the axioms of efficiency and additivity would bridge explainability and control theory, offering mathematically grounded interpretability for sequential or dynamical data.
    \item[5.] \textbf{Efficient computation in high dimensions.} Develop reduced-order, sampling-based, or variational approximations for vector-valued SHAP values that preserve the fundamental axioms of fairness and efficiency while remaining computationally tractable in large-scale models. This issue has been partially addressed in \cite{morales2025shap}, where spectral representations were proposed to approximate SHAP computations while retaining theoretical consistency. However, extending such approaches to the vector-valued setting and establishing rigorous error bounds remain open problems.
\end{itemize}


\bibliographystyle{acm}
\bibliography{biblio.bib}

\appendix 
\section{Proof of the main results}\label{appendix:proofs:main:results}

This appendix collects the complete proofs of the results presented in Section \ref{sec:SHAP_math}. For clarity, proofs are presented in a self-contained manner, following the logical order of the statements in the main text.

\begin{proof}[Proof of Theorem \ref{thm:coordinatewise_decomposition}]
Let $(e_k)_{k=1}^m$ be the canonical basis of $\mathbb{R}^m$. For a scalar game $g\in\MC G_{[n]}$, define the vector-valued game supported on the $k$-th coordinate by
\begin{align*}
    (\iota_k g)(S)\coloneqq g(S)e_k, \quad\text{ for all } S\subseteq[n].    
\end{align*}
Moreover, for a vector-valued game $v\in \MC G^m_{[n]}$, define its $k$-th scalar coordinate game 
\begin{align*}
    v^k\coloneqq \pi_k\circ v\in \MC G_{[n]}.    
\end{align*}

\medskip 
\noindent\textit{Step 1: no leakage across output coordinates.} Fix $k\in[m]$. We claim that for every scalar game $g\in\MC G_{[n]}$, every player $i\in[n]$, and every $\ell\neq k$,
\begin{align}\label{eq:no_leakage}
    \pi_\ell\big(\Phi_i(\iota_k g)\big)=0.    
\end{align}

Because $\Phi$ is additive (Axiom (iv) in Definition \ref{def:SHAP_operator_vector}), it suffices to prove \eqref{eq:no_leakage} for a spanning family of scalar games. Here we use the family of unanimity games  
\begin{align*}
    (u_T)_{\varnothing\neq T\subseteq[n]}, \quad u_T(S)\coloneqq\mathbf{1}_{T\subseteq S},    
\end{align*}
with $\mathbf{1}_T$ denoting the characteristic function of $T$.

\medskip 
\noindent Fix $\varnothing\neq T\subseteq[n]$ and set $v_T=\iota_k u_T$.
\begin{itemize}
    \item If $i\notin T$, then $v_T(S\cup\{i\})=v_T(S)$ for all $S\subseteq([n]\setminus\{i\})$, so $i$ is a dummy player for $v_T$. By the dummy Axiom (iii) in Definition \ref{def:SHAP_operator_vector}, $\Phi_i(v_T)=0$.
    \item If $i,j\in T$, then $v_T(S\cup\{i\})=v_T(S\cup\{j\})$ for all $S\subseteq([n]\setminus\{i,j\})$, so $i$ and $j$ are symmetric in $v_T$. By the symmetry Axiom (ii) in Definition \ref{def:SHAP_operator_vector}, $\Phi_i(v_T)=\Phi_j(v_T)$.
\end{itemize}
Hence there exists some $a\in\mathbb{R}^m$ such that
\begin{align*}
    \Phi_i(v_T)=a \text{ for } i\in T \quad\text{ and }\quad \Phi_i(v_T)=0 \text{ for } i\notin T.    
\end{align*}
Now, applying the efficiency Axiom (i) in Definition \ref{def:SHAP_operator_vector}, we get
\begin{align*}
    \sum_{i=1}^n \Phi_i(v_T)=v_T([n])=u_T([n])e_k=e_k.    
\end{align*}
On the other hand, we also have that 
\begin{align*}
    \sum_{i=1}^n \Phi_i(v_T)=\sum_{i\in T} a=|T|a.    
\end{align*}

Thus $a=|T|^{-1}e_k$, which has no $\ell$-th component for $\ell\neq k$. Therefore $\pi_\ell(\Phi_i(v_T))=0$ for $\ell\neq k$. This proves \eqref{eq:no_leakage} for the games $(u_T)_{\varnothing\neq T\subseteq[n]}$, hence for all $g\in \MC G_{[n]}$ by additivity.

For each coordinate $k\in[m]$, the projection of $\Phi$ onto the $k$-th coordinate defines a scalar operator $\Phi^{(k)}: \mathcal{G}_{[n]}\to \mathbb{R}^n$. By projecting axioms (i)-(iv) onto coordinate $k$, we deduce that $\Phi^{(k)}$ satisfies the classical Shapley axioms. By the scalar uniqueness, $\Phi^{(k)}$ must coincide with the scalar Shapley operator. Hence all coordinates use the same scalar operator, which establishes the component-wise decomposition.

\medskip 
\noindent\textit{Step 2: definition of the scalar operator and decomposition.} Fix $k\in[m]$. Using Step 1, for each scalar game $g\in \MC G_{[n]}$ and each player $i$, the vector $\Phi_i(\iota_k g)$ lies in $\mathrm{span}\{e_k\}$. Hence there exists a unique scalar $\phi_i(g)$ such that
\begin{align}\label{eq:game_decomposed}
    \Phi_i(\iota_k g)=\phi_i(g)e_k.    
\end{align}

Define $\phi(g)\coloneqq (\phi_1(g),\ldots,\phi_n(g))\in\mathbb{R}^n$. By construction, $\phi: \MC G_{[n]}\to\mathbb{R}^n$ is additive (it inherits this from $\Phi$). Now take an arbitrary $v\in \MC G^m_{[n]}$. For each coalition $S\subset[n]$, we have
\begin{align*}
    v(S)=\sum_{k=1}^m \pi_k(v(S))e_k \quad\Longleftrightarrow\quad v=\sum_{k=1}^m \iota_k(v^{(k)}).    
\end{align*}
Using the additivity of $\Phi$ and \eqref{eq:game_decomposed}, we then conclude that
\begin{align*}
    \Phi_i(v) = \sum_{k=1}^m \Phi_i(\iota_k(v^{(k)})) = \sum_{k=1}^m \phi_i(v^{(k)})e_k = \big(\phi_i(v^{(1)}),\dots,\phi_i(v^{(m)})\big).    
\end{align*}
\end{proof}

Theorem \ref{thm:coordinatewise_decomposition} shows that the Shapley axioms rigidly constrain the vector-valued case: any admissible attribution operator must act coordinate-wise, so no coupling between output components is compatible with efficiency, symmetry, dummy, and additivity. This structural fact is the key ingredient for the uniqueness Theorem \ref{thm:existence:uniqueness:Shapley:formula}. In preparation to the proof of this result, the following technical Lemmas are required.

\begin{lemma}\label{lemma:decomposition:GNm}
The vector space $\mathcal{G}_{[n]}^m$ decomposes as a direct sum of $m$ scalar subspaces:
\begin{align*}
    \mathcal{G}_{[n]}^m= \mathcal{G}_{[n]}^{(1)} \oplus \ldots \oplus \mathcal{G}_{[n]}^{(m)},
\end{align*}
where for each $j\in m$, $\mathcal{G}_{[n]}^{(j)}$ is defined by 
\begin{align*}
    \mathcal{G}_{[n]}^{(j)}\coloneqq\big\{v\in \mathcal{G}_{[n]}^m \,:\, v^{(k)}\equiv 0\text{ for }k\neq j\big\}.
\end{align*}    
\end{lemma}

\begin{proof}
For any $v\in \mathcal{G}_{[n]}^m$, define $v^{[j]}\in \mathcal{G}_{[n]}^{(j)}$ by 
\begin{align*}
    v^{[j]}(S)=\Big(0,\ldots, 0,v^{(j)}(S),0,\ldots, 0\Big).
\end{align*}
Then, it follows that 
\begin{align*}
    v=\sum_{j=1}^{m} v^{[j]}\quad\text{ and }\quad v^{[j]}(\varnothing)=0.
\end{align*}
Let us define
\begin{align*}
    \sum_{k\neq j} \mathcal{G}_{[n]}^{(k)}:=\left\{ \sum_{k\neq j} w^{(k)}\,:\, w^{(k)}\in \mathcal{G}_{[n]}^{(k)} \right\}.
\end{align*}

If $w\in \mathcal{G}_{[n]}^{(j)} \cap \sum_{k\neq j} \mathcal{G}_{[n]}^{(k)}$, then $w^{(j)}\equiv 0$ for $k\neq j$. Hence, $w\equiv 0$ and thus the sum is direct. Therefore, we have
\begin{align*}
    \mathcal{G}_{[n]}^m=\mathcal{G}_{[n]}^{(1)} \oplus \ldots \oplus \mathcal{G}_{[n]}^{(m)}.
\end{align*}
\end{proof}

This structural result has an immediate conceptual consequence: every vector-valued game can be viewed as a superposition of $m$ scalar games, each corresponding to one output coordinate. Leveraging this decomposition, we can now study the class of value operators that satisfy the axioms of efficiency, symmetry, dummy, and additivity, and determine how these operators act across coordinates.

\begin{lemma}\label{lemma:coordinates:Shapley}
Let $\Phi: \mathcal{G}_{[n]}^m \to \R^{m\cdot n}$ be a linear operator satisfying axioms (i)-(iv) in Definition \ref{def:SHAP_operator_vector}. Then, there exist linear mappings $\Psi^{(j)}: \mathcal{G}_{[n]} \to \R^n$ such that for all $v\in \mathcal{G}_{[n]}^m$ and $i\in [n]$, we have
\begin{align*}
    \Phi_i(v):=\left( \Psi_i^{(1)}(v^{(1)}),\ldots, \Psi_i^{(m)}(v^{(m)}) \right).
\end{align*}
Moreover, each $\Psi^{(j)}$ satisfies the scalar Shapley axioms of Definition \ref{def:SHAP_operator}.
\end{lemma}

\begin{proof}
From Lemma \ref{lemma:decomposition:GNm} and additivity, it follows that 
\begin{align*}
    \Phi(v)=\sum_{j=1}^{m} \Phi(v^{[j]}).
\end{align*}
Notice that each $\Phi(v^{[j]})$ has nonzero entries only in coordinate $j$. Now, we set 
\begin{align*}
    \Psi^{(j)}(u):=(\pi_j \circ \Phi \circ \tau_j)(u),
\end{align*}
where $\tau_j:\mathcal{G}_{[n]} \to \mathcal{G}_{[n]}^{(j)}$ embeds a scalar game $u$ into the $j$-th coordinate, and $\pi_j:\R^{mn} \to \mathbb{R}^n$ projects onto coordinate $j$. Then, it is clear that for each $i\in [n]$, we can write $\Phi_i$ in the form:
\begin{align*}
    \Phi_i(v):=(\Psi_i^{(1)}(v^{(1)}),\ldots, \Psi_i^{(m)}(v^{(m)})).
\end{align*}

Projecting the vectorial axioms (i)-(iv) of Definition \ref{def:SHAP_operator_vector} onto coordinate $j$ shows that each $\Psi^{(j)}$ satisfies the scalar axioms of efficiency, symmetry, dummy, and additivity.
\end{proof}

\begin{proof}[Proof of Theorem \ref{thm:existence:uniqueness:Shapley:formula}]

The proof is divided into two steps:

\noindent {\bf Existence.} For each $i\in [n]$, define $\phi_i(v)$ by \eqref{def:Phi:i:vector}. We verify the axioms of Definition \ref{def:SHAP_operator_vector}.
\begin{itemize}
    \item The additivity follows because the definition \eqref{def:Phi:i:vector} involves finite sums of linear functionals of $v$. Therefore, for all $i\in[n]$, $v,w\in \mathcal{G}_{[n]}^m$, and $\alpha,\beta\in \mathbb{R}$, we have 
    \begin{align*}
        \Phi_i(\alpha v+ \beta w)=\alpha \Phi_i(v) + \beta \Phi_i(w).
    \end{align*}

    \item For any permutation $\sigma \in \mathfrak{S}_{[n]}$ and permuted game $v^{\sigma}(S)=v(\sigma^{-1}(S))$, it follows that 
    \begin{align*}
        P_{\sigma(i)} (\sigma \circ \pi)=\sigma (P_i(\pi)),
    \end{align*}
    so that $\phi_{\sigma(i)}(v^{\sigma})=\phi_i(v)$, i.e., Symmetry follows.
    \item Dummy is direct.
    \item Fix $\pi \in \mathfrak{S}_{[n]}$. Summing over $i\in [n]$ yields a telescoping sum
    \begin{align}\label{telecopy:sum}
        \sum_{i\in [n]} \Big[v(P_i(\pi)\cup \{i\}) - v(P_i(\pi))\Big]=v([n])-v(\varnothing)=v([n]).
    \end{align}
    Summing over all $\pi$ gives 
    \begin{align*}
        \sum_{i\in [n]} \phi_i(v)=v([n]),
    \end{align*}
    i.e., efficiency holds. 
\end{itemize}

\noindent {\bf Uniqueness.} Let $\widetilde{\Psi}$ be any operator satisfying (i)-(iv). By Lemma \ref{lemma:coordinates:Shapley}, we know that each coordinate $\widetilde{\Psi}_i$ can be written in the form:
\begin{align*}
    \widetilde{\Phi}_i(v)=\left( \Psi_i^{(1)}(v^{(1)}),\ldots , \Psi_i^{(m)}(v^{(m)}) \right),
\end{align*}
for some linear maps $\Psi_i^{(j)}: \mathcal{G}_{[n]} \to \mathbb{R}^{[n]}$. Each $\Psi^{(j)}$ satisfies the scalar Shapley axioms, hence by Theorem \ref{thm:SHAP_scalar} it follows that
\begin{align*}
    \Psi_i^{(j)}(v^{(j)})=\dfrac{1}{n!} \sum_{\pi \in \mathfrak{S}_{[n]}} \left[ v^{(j)}(P_i(\pi)\cup \{i\}) - v^{(j)}(P_i(\pi)) \right].
\end{align*}
Collecting coordinates gives exactly the formula \eqref{def:Phi:i:vector}. Therefore, $\widetilde{\Phi}=\Phi$.
\end{proof}

To conclude, we verify that the SHAP values defined in the multi-output setting satisfy the expected reconstruction property. This result follows directly from vector efficiency applied to the characteristic function induced by the predictor and formalizes the exact decomposition of the prediction into feature-wise contributions.

\begin{proof}[Proof of Proposition \ref{corollary:equality:efficiency}]
For a measurable function $f:\R^n \to \R^m$, define $v_f(S)$ as in \eqref{def:v:f:R:m}. We can immediately see that for the empty and full coalitions it holds
\begin{align*}
    v_f(\varnothing)=0\quad\text{ and }\quad v_f([n])=f(x)- \mathbb{E}_{X\sim \mu}[f(X)].
\end{align*}

Appplying Theorem \ref{thm:existence:uniqueness:Shapley:formula} to the game $\tilde{v}_f$ yields
\begin{align*}
    \sum_{i=1}^n \phi_i(f;x)=v_f([n])=f(x)-\mathbb{E}_{X \sim \mu}[f(X)].
\end{align*}



\end{proof}

\section{The role of correlation in SHAP computations}\label{app:B}

The theoretical analysis developed in this paper establishes that, under the Shapley axioms, feature attribution for multi-output models necessarily decomposes component-wise across output coordinates. This result constrains the structural form of admissible explanations, but it does not preclude the presence of strong statistical dependencies among input features or among output variables, which are ubiquitous in practical learning problems.

In this appendix, we investigate the role of input and output correlations in a simplified and analytically tractable setting, focusing on linear multi-output models. This choice allows us to explicitly characterize how correlations in the input space affect conditional expectations and marginal contributions, and how correlations among outputs influence learning efficiency and computational behavior, without contradicting the axiomatic rigidity results established in Section \ref{sec:SHAP_math}. The analysis clarifies that correlation impacts interpretability indirectly, through modeling and computation, rather than through the structure of the attribution itself.

To see concretely how these mechanisms arise, we begin by examining the effect of input correlations on the characteristic function $v_f(S)$ defined in \eqref{def:v:f:R:m} and the resulting conditional expectations that underpin SHAP values.

This characteristic function depends critically on the joint distribution of the input variables. When the features $(X_1,\dots,X_n)$ are statistically independent, conditional expectations factorize, leading to simple, coordinate-wise attributions. In contrast, when features are correlated (as it happens in most realistic datasets) conditioning on one variable modifies the distribution of the others, and this dependence propagates into the SHAP values, altering the way contributions are assigned across features.

To illustrate these effects clearly, we restrict attention to a simple yet analytically tractable setting: linear models with Gaussian inputs. Specifically, we consider
\begin{align}\label{eq:linear}
    f:\R^n\to\R^m, \quad f(x)=B_0 + B^\top x \quad\text{ with } B_0\in\R^m \text{ and } B\in\R^{n\times m},
\end{align}
and assume that
\begin{align*}
    X=(X_1,\ldots,X_n)\sim\MC N(\mu,\Sigma)
\end{align*}
follows a multivariate normal distribution with mean vector $\mu = \mathbb{E}[X] = (\mu_1,\dots,\mu_n)^\top\in\R^n$ and positive-definite covariance matrix $\Sigma\in\R^{n\times n}$. Within this unified Gaussian framework, the independence or correlation of features is encoded entirely in the structure of $\Sigma$: 
\begin{itemize}
    \item independent features correspond to a diagonal covariance matrix, $\Sigma = \text{diag}(\sigma_1^2, \ldots, \sigma_n^2)$;
    \item correlated features arise when $\Sigma$ contains non-zero off-diagonal entries, meaning that at least some variables exhibit linear dependence.
\end{itemize}

\subsection{Independent features}

When $\Sigma$ is diagonal, each coordinate $X_i$ varies independently of the others. In this case, the conditional expectation that defines $v_f(S)$ simplifies, as fixing $X_S = x_S$ does not alter the distribution of the remaining features. Consequently, the SHAP value of each feature depends only on its own coordinate. In more detail, we can prove the following result.
\begin{proposition}\label{prop:SHAP_linear}
Let $f:\R^n\to\R^m$ be the linear predictor defined in \eqref{eq:linear}. Assume that the input vector $X\sim\MC N(\mu,\Sigma)$ has diagonal covariance $\Sigma=\emph{diag}(\sigma_1^2,\dots,\sigma_n^2)$, so that the components of $X$ are mutually independent. Then, for every $x=(x_1,\ldots,x_n)\in\R^n$ and $i\in[n]$, we have
\begin{align*}
    \phi_i(f;x)=B_i(x_i-\mu_i)\in\R^m,
\end{align*}
where $B_i\in\R^m$ is the $i$-th row of the matrix $B$. 
\end{proposition}

\begin{proof}
Under the Gaussian assumption with diagonal covariance, conditioning on $X_S=x_S$ leaves all other coordinates independent and centered at their means, that is, 
\begin{align*}
    \mathbb{E}[X_j\mid X_S=x_S] = \begin{cases} x_j, & j\in S \\ \mu_j, & j\notin S \end{cases}.
\end{align*}
Consequently, 
\begin{align*}
    v_f(S) = \mathbb{E}[f(X)\mid X_S=x_S] = B_0+B^{\top}\mu+\sum_{j\in S}B_j(x_j-\mu_j).
\end{align*}
In particular, for any subset $S\subseteq[n]\setminus\{i\}$, we get the expression
\begin{align*}
    v_f(S\cup\{i\})-v_f(S)=B_i(x_i-\mu_i),    
\end{align*}
which is independent of $S$. Averaging this marginal contribution over all subsets with the SHAP weights 
\begin{align*}
    w(S)=\frac{|S|!(n-|S|-1)!}{n!}     
\end{align*}
yields directly
\begin{align*}
    \varphi_i(f;x)=B_i(x_i-\mu_i).    
\end{align*}
\end{proof}

Proposition \ref{prop:SHAP_linear} shows that, under independence, the SHAP value of each feature depends exclusively on its own variable. In other words, the attribution $\varphi_i(f;x)$ is determined solely by the deviation of $x_i$ from its mean $\mu_i$, and is unaffected by the values of the remaining features. Each coordinate thus acts independently, contributing linearly to the model output through its direct effect. In this setting, the SHAP decomposition is purely local and reproduces the standard additive structure of the linear model.

\subsection{Correlated features}

The situation illustrated in the previous Section changes when the features $X\sim\MC N(\mu,\Sigma)$ are correlated, i.e. the covariance $\Sigma$ is a full positive-definite matrix. For any subset $S\subseteq [n]$, let us define:
\begin{itemize}
    \item $\mu_S\in \R^{|S|}$ as the restriction of $\mu$ to the coordinates in $S$;
    \item $\Sigma_{S,S}\in \R^{|S|\times |S|}$ as the principal submatrix of $\Sigma$ indexed by $S$;
    \item $\Sigma_{:,S}\in \R^{n\times |S|}$ as the submatrix consisting of the columns of $\Sigma$ indexed by $S$.
\end{itemize}
Moreover, define the conditional expectation matrix as follows:
\begin{align}\label{eq:conditional}
    A_S\coloneqq \Sigma_{:,S} \Sigma_{S,S}^{-1}\in \R^{n\times |S|}.
\end{align}
We have the following result.

\begin{proposition}\label{prop:SHAP_linear_correlated}
Let $f:\R^n\to\R^m$ be the linear predictor defined in \eqref{eq:linear}. Assume that the input vector $X\sim\MC N(\mu,\Sigma)$ has a full, positive-definite covariance matrix, possibly with non-zero off-diagonal entries, so that the components of $X$ are correlated. For each subset $S\subseteq[n]$, define $A_S\in \R^{n\times |S|}$ as in \eqref{eq:conditional} and let $\widehat{A}_R\in\R^{n\times n}$ denote the matrix obtained by embedding $A_S$ into the full coordinate space, inserting zeros in the columns outside $S$. For each feature index $i\in [n]$, define the matrix 
\begin{align*}
    M_i(\Sigma)\coloneqq \sum_{S\subseteq ([n]\setminus \{i\})} w(S) \Big(\widehat{A}_{S\cup \{i\}} - \widehat{A}_S\Big) \in \R^{n\times n},
\end{align*}
where 
\begin{align*}
    w(S)=\frac{|S|!(n-|S|-1)!}{n!}     
\end{align*}
are the SHAP weights in formula \eqref{formula:Shapley:values:mdim}. Then, for every $x=(x_1,\ldots,x_n)\in\R^n$ and $i\in[n]$, the SHAP value $\phi_i(f;x)$ admits the exact representation
\begin{align*}
    \phi_i(x)=B^\top M_i(\Sigma)(x-\mu)\in\R^m.
\end{align*}
\end{proposition}

\begin{proof}
For any subset $S\subseteq [n]$ the conditional mean of $X$ given $X_S=x_s$ is linear, i.e., 
\begin{align*}
    \mathbb{E}[X|X_S=x_s]=\mu + A_S(x_s-\mu_S).
\end{align*}
Using the full embedding $\widehat{A}_S$ we equivalently write 
\begin{align*}
    \mathbb{E}[X|X_S=x_S]=\mu + \widehat{A}_S(x-\mu).
\end{align*}

Therefore, the conditional expected model value is
\begin{align*}
    \mathbb{E} [f(X)|X_S=x_S]=B_0 + B^\top \mu + B^\top \widehat{A}_S(x-\mu).
\end{align*}


On the other hand, the conditional expectation is given as
\begin{align*}
    \mathbb{E}_{X\sim \mu} [f(X)]=B_0+B^\top \mu,
\end{align*}
so that $v_f(S)=B^\top \hat{A}_{S}(x-\mu)$. Moreover, for each $i\in [n]$ and $S\subseteq [n] \setminus \{i\}$, we have
\begin{align*}
    v_f(S\cup \{i\}) - v_f(S) = B^\top\Big(\widehat{A}_{S\cup \{i\}}(x-\mu) - \widehat{A}_S(x-\mu)\Big) = B^\top \Big(\widehat{A}_{S\cup \{i\}} - \widehat{A}_S\Big)(x-\mu).
\end{align*}
This immediately yields
\begin{align*}
    \phi_i(f;x) &= \sum_{S\subseteq ([n]\setminus \{i\})} w(S) \Big(v(S\cup \{i\}) - v(S)\Big) 
    \\
    &= B^\top \left(\sum_{S\subseteq ([n]\setminus \{i\})} w(S) \Big(\widehat{A}_{S\cup \{i\}} - \widehat{A}_S\Big) \right) (x-\mu) = B^\top M_i(\Sigma) (x-\mu).
\end{align*}
\end{proof}

In this correlated setting, the SHAP value of each feature is no longer confined to its own coordinate but depends on all components of the vector $x-\mu$. Indeed, in coordinates, one may rewrite
\begin{align*}
    \phi_i(f;x) = \sum_{j=1}^n\sum_{k=1}^n B_j\big(M_i(\Sigma)\big)_{j,k}(x_k-\mu_k),
\end{align*}
where $B_j$ denotes the $j$-th row of $B$. 

In particular, the matrix $M_i(\Sigma)$ redistributes the total contribution of each feature according to the covariance structure: the diagonal entries $(M_i(\Sigma))_{i,i}$ capture the direct effect of the $i$-th feature, while the off-diagonal terms encode indirect influences arising from correlations among features. As a consequence, the attribution $\phi_i(f;x)$ aggregates information not only from $x_i$ but also from correlated variables $x_k$, revealing how dependence in the input distribution couples the feature contributions.

\subsection{Correlated outputs and implications for Shapley-based explanations}

In many applications, particularly in biomedical settings, output variables are driven by shared physiological mechanisms and therefore exhibit non-trivial correlations. In such cases, the learning problem naturally acquires a multi-objective structure \cite{Jahn2009VectorOptimization,Miettinen1999NonlinearMultiobjectiveOptimization}, where each output corresponds to a distinct objective and training seeks a parameter configuration that performs well across all of them. From this perspective, minimizing a global loss in a multi-output model can be interpreted as the search for a weak Pareto-optimal solution of an underlying multi-objective optimization problem. The shared representation learned by the network captures common latent factors influencing all outputs, leading to more efficient training and improved sample efficiency compared with training multiple independent single-output models.

These considerations also have implications for the computation of SHAP values in multi-output settings. As shown in Section \ref{sec:SHAP_math} and Appendix \ref{appendix:proofs:main:results}, feature attributions for mappings $f:\mathbb{R}^n \to \mathbb{R}^m$ decompose component-wise, independently of the statistical relationship between the outputs. Nevertheless, when outputs are strongly correlated, it is natural to adopt a joint multi-output architecture. While this modeling choice does not alter the definition of Shapley values, it allows both training and explanation procedures to be carried out within a single model, avoiding the duplication of computations inherent in independent single-output approaches.

\end{document}